\documentclass{article} % For LaTeX2e
\usepackage[final]{colm2026_conference}

\usepackage{microtype}
\usepackage{hyperref}
\usepackage{url}
\usepackage{booktabs}

\usepackage{lineno}

\usepackage{times}
\usepackage{latexsym}
\usepackage{amsmath}
\usepackage{amsthm}
\usepackage{amssymb}
\usepackage{booktabs} 
\usepackage{graphicx}
\usepackage{subcaption}

\PassOptionsToPackage{table}{xcolor}
\usepackage{colortbl}
\usepackage{multirow}
\usepackage{tcolorbox}
\usepackage{wrapfig}
\usepackage{adjustbox}
\usepackage{enumitem}

\newtheorem{theorem}{Theorem}
\newtheorem{lemma}{Lemma}

\newtheorem{assumption}{Assumption}
\newtheorem{remark}{Remark}

\usepackage{setspace}
\AtBeginDocument{%
  \addtolength\abovedisplayskip{-0.3\baselineskip}%
  \addtolength\belowdisplayskip{-0.3\baselineskip}%
  \addtolength\abovedisplayshortskip{-0.3\baselineskip}%
  \addtolength\belowdisplayshortskip{-0.3\baselineskip}%
}

\usepackage{titlesec}
\titlespacing*{\section}{0pt}{0.15\baselineskip}{0.15\baselineskip}
\titlespacing*{\subsection}{0pt}{0.15\baselineskip}{0.15\baselineskip}

\definecolor{lightred}{RGB}{254, 232, 231}
\definecolor{lightpurple}{RGB}{240, 240, 253}
\definecolor{lightblue}{RGB}{219, 244, 255}
\definecolor{lightorange}{RGB}{255, 243, 229}

\definecolor{darkblue}{rgb}{0, 0, 0.5}
\hypersetup{colorlinks=true, citecolor=darkblue, linkcolor=darkblue, urlcolor=darkblue}

\newcommand{\name}{\textsc{Rims}}
\title{{\name}: Preference Optimization via Smoothed Multi-pair Aggregation for Small-Scale LLM Retrieval-Augmented Generation}

\author{Pei Tian$^{1}$, Zihan Dong$^{2}$, Tianci Liu$^3$, Linjun Zhang$^2$, Haoyu Wang$^4$ \\
$^1$ Department of Computer Science, 
Columbia University\\
$^2$ Department of Statistics, Rutgers University \\
$^3$ School of Electrical and Computer Engineering, Purdue University \\
$^4$ Departiment of Computer Science, SUNY Albany \\
\texttt{pt2632@columbia.edu;hwang28@albany.edu} \\
}

\begin{document}

\ifcolmsubmission
\linenumbers
\fi

\maketitle

\begin{abstract}
Small-scale language models (SLMs) are attractive for retrieval-augmented generation (RAG) in resource-constrained settings, but their limited capacity makes them highly sensitive to noisy or spurious retrieved evidence. Existing preference-based methods such as RoseRAG select only the hardest single preference pair via hard argmin/argmax, discarding the remaining signal; others treat multiple pairs as independent binary comparisons, resulting in low data utilization. We propose {\name}, a three-stage preference optimization framework comprising (1) synthetic chain-of-thought preference data generation via rejection sampling using the target SLM itself without relying on proprietary models, (2) a differentiable soft aggregation mechanism that replaces hard selection with a smooth operator, preserving gradient signal from all preference pairs while retaining the discriminative structure of margin-aware selection, and (3) preference optimization with the smoothed objective applied to multiple alignment algorithms. We theoretically show that the smoothed approximation admits a controllable error bound and that smooth aggregation yields provably tighter gradient alignment to the oracle objective than hard selection.
Experiments on four multi-hop question answering benchmarks show that our approach outperforms state-of-the-art baselines across multiple SLM backbones, achieving consistent gains in Exact Match and F1 under noisy retrieval conditions. Our implementation is available at \url{https://github.com/tptrix29/RIMS}.

\end{abstract}

\section{Introduction}

Large language models (LLMs) \citep{achiam2023gpt, comanici2025gemini, yang2025qwen3, grattafiori2024llama} have achieved strong performance across many natural language tasks by learning vast amounts of linguistic and world knowledge during pretraining. However, their high computational and memory requirements limit practical deployment in resource-constrained environments such as on-device systems or low-infrastructure settings \citep{zhou2024surveyefficientinferencelarge, xu2024ondevicelanguagemodelscomprehensive}. To address this, small-scale language models (SLMs) with dramatically fewer parameters have become a promising alternative due to their lower inference cost, portability, and ease of deployment \citep{belcak2025small, lu2025smalllanguagemodelssurvey, vernikos2024smalllanguagemodelsimprove}. 

Despite their computational efficiency, SLMs inherently cannot internalize as much knowledge as larger models, and their reduced capacity often leads to weaker semantic understanding and generalization on knowledge-intensive tasks \citep{ovadia2024fine}. Retrieval-augmented generation (RAG) \citep{lewis2021retrievalaugmentedgenerationknowledgeintensivenlp,gao2023retrieval,jiang2023active} offers a natural remedy by coupling the model with an external knowledge source, allowing it to answer queries that require up-to-date or domain-specific information beyond its pretrained parameters. Applying RAG to SLMs can substantially enhance performance by offloading factual grounding to an external retrieval component \citep{wang2023selfknowledgeguidedretrievalaugmentation, asai2024selfrag, xu2025simragselfimprovingretrievalaugmentedgeneration}. Yet even with RAG, robust integration of retrieved content remains challenging for small models: their limited capacity makes it difficult to distinguish relevant evidence from distractors, to reason over multiple retrieved passages, and to suppress noisy or contradictory context during generation \citep{yoran2024makingretrievalaugmentedlanguagemodels, liu2025roserag}. This underscores the need for mechanisms that go beyond retrieval quality and directly improve how SLMs learn to utilize imperfect evidence.

Existing methods for improving SLM robustness under RAG fall into three broad categories. \textit{Prompt-based approaches} \citep{wang2025astuteragovercomingimperfect,wang2024blendfilter} instruct the model to disregard irrelevant retrieved content, but their effectiveness hinges on the model's own ability to identify noise, an assumption that often fails for SLMs. \textit{Supervised fine-tuning (SFT)} \citep{yoran2024makingretrievalaugmentedlanguagemodels, wei2025instructraginstructingretrievalaugmentedgeneration} aligns models with curated instruction--response pairs, yet is highly sensitive to noise in retrieved contexts and prone to overfitting on imperfect evidence \citep{chu2025sftmemorizesrlgeneralizes}, particularly for capacity-limited SLMs. A third line of work, \textit{preference optimization} \citep{rafailov2024directpreferenceoptimizationlanguage,xu2024contrastivepreferenceoptimizationpushing,hong2024orpomonolithicpreferenceoptimization,ethayarajh2024kto,meng2024simpo}, offers a more flexible alternative by learning from relative quality signals rather than fixed targets~\citep{liu2025roserag,chen2024noisecontrastivealignmentlanguage}. However, effectively constructing and utilizing preference data under noisy RAG conditions remains an open challenge.

In RAG settings, preference optimization is particularly appealing because it learns from relative quality comparisons rather than treating any single response as ground truth, making it naturally more robust to noisy retrieved contexts than supervised fine-tuning. Moreover, by sampling multiple responses per query through techniques such as rejection sampling with chain-of-thought prompting, one can readily obtain multiple preferred and non-preferred candidates, yielding rich multi-pair preference data for each training instance. Despite the abundance of such supervision, existing methods utilize only a small fraction of it. Concretely, when multiple preference pairs are available for a single query, RoseRAG~\citep{liu2025roserag} selects only the single hardest pair based on the initial model's likelihoods, discarding the remaining signal entirely and fixing the selection throughout training even as the policy evolves. NCA~\citep{chen2024noisecontrastivealignmentlanguage} aggregates multiple negative responses but does not jointly model positive candidates, resulting in an incomplete use of the available preference structure. In both cases, the rich multi-pair preference information is substantially underutilized, motivating the need for a principled mechanism that can fully leverage all available preference pairs during optimization.
% A key practical challenge is that generating synthetic preference data is computationally expensive, yet existing methods utilize only a small fraction of it. Concretely, when multiple preference pairs are available for a single query, RoseRAG \citep{liu2025roserag} selects only the single hardest pair via hard argmin/argmax based on the initial model's likelihoods, discarding the remaining signal entirely and fixing the selection throughout training even as the policy evolves. Noise Contrastive Alignment (NCA) \citep{chen2024noisecontrastivealignmentlanguage} aggregates multiple negative responses but does not jointly model positive candidates, resulting in an incomplete use of the available preference structure. In both cases, the rich multi-pair preference information is substantially underutilized.

% \begin{wrapfigure}{r}{0.6\columnwidth}
% \centering
%   \includegraphics[width=0.58\columnwidth]{figures/RoseRAG_vs_RIMS_robust.pdf}
%   \caption{Qualitative comparison on a multi-hop question. RoseRAG, which relies on hard pair selection, is misled by noisy retrieved passages and produces an incorrect answer. {\name}, using soft aggregation over multiple preference pairs, arrives at the correct answer despite the same noisy context.}
%   \label{fig:roserag_prism_comparison}
% \end{wrapfigure}

To address these limitations, we propose a p\textbf{r}eference opt\textbf{i}mization framework via smoothed \textbf{m}ulti-pair aggregation for \textbf{S}LM retrieval-augmented generation~({\name}). Given multiple preference pairs generated per query through chain-of-thought (CoT)~\citep{wei2023chainofthoughtpromptingelicitsreasoning} rejection sampling, {\name} introduces a differentiable soft aggregation mechanism that replaces hard pair selection with a smooth, temperature-controlled operator. This operator consolidates all positive and negative candidates into a unified pseudo-pair, preserving gradient signal from every candidate while retaining the discriminative structure of margin-aware selection. The smoothed objective is general and can be applied to multiple preference optimization algorithms~\citep{rafailov2024directpreferenceoptimizationlanguage,xu2024contrastivepreferenceoptimizationpushing,hong2024orpomonolithicpreferenceoptimization}. By adaptively aggregating gradient signals across multiple samples of the same question, {\name} improves the model's ability to identify and integrate relevant evidence from noisy retrieval contexts. We further provide theoretical guarantees that the smoothed objective admits a controllable approximation error and that its gradient updates are provably closer to the oracle than those of hard selection.

% \begin{figure*}[t]
% \centering
%   \includegraphics[width=0.6\columnwidth]{figures/RoseRAG_vs_RIMS.pdf}
%   \caption {Qualitative comparison on a multi-hop question. RoseRAG, which relies on hard pair selection, is misled by noisy retrieved passages and produces an incorrect answer. {\name}, using soft aggregation over multiple preference pairs, arrives at the correct answer despite the same noisy context.}
%   \label{fig:roserag_prism_comparison}
% \end{figure*}

Our contributions can be summarized as follows: 
% 1)~We propose {\name}, a preference optimization framework for SLM-based RAG that generates multi-pair CoT preference data via rejection sampling using the target SLM itself, without relying on proprietary models. 2)~We introduce a differentiable soft aggregation mechanism that consolidates multiple preference pairs into a unified, algorithm-agnostic objective, replacing hard pair selection while preserving the discriminative structure of margin-aware optimization. 3)~We provide theoretical analysis showing that the LSE-based smoothing admits a controllable approximation error bound, and that smooth aggregation yields provably tighter gradient alignment to the oracle objective than hard selection under noisy reward estimation. 4)~Extensive experiments on four multi-hop question answering benchmarks with two SLM backbones show that {\name} consistently outperforms state-of-the-art baselines.
\begin{itemize}[leftmargin=1em]
    \item We propose {\name}, a preference optimization framework for SLM-based RAG that generates multi-pair CoT preference data via rejection sampling using the target SLM itself, without relying on proprietary models.
    \item We introduce a differentiable soft aggregation mechanism that consolidates multiple preference pairs into a unified, algorithm-agnostic objective, replacing hard pair selection while preserving the discriminative structure of margin-aware optimization.
    \item We provide theoretical analysis showing that the LSE-based smoothing admits a controllable approximation error bound, and that smooth aggregation yields provably tighter gradient alignment to the oracle objective than hard selection under noisy reward estimation.
    \item Experiments on four multi-hop question answering benchmarks with two SLM backbones show that {\name} consistently outperforms state-of-the-art baselines.
\end{itemize}

\section{Related Work}
Retrieval-augmented generation (RAG) has emerged as an effective paradigm for alleviating the knowledge limitations of language models by grounding generation in external corpora. While early RAG frameworks \citep{izacard2022atlasfewshotlearningretrieval} were primarily designed for large language models (LLMs), recent studies \citep{wang2023selfknowledgeguidedretrievalaugmentation, li2024llatrievalllmverifiedretrievalverifiable, yu2024rankrag} have explored adapting RAG to SLMs to improve efficiency and deployability in resource-constrained settings. However, due to limited representational capacity, SLMs are particularly vulnerable to noisy or irrelevant retrieved contexts, often leading to degraded reasoning and hallucinated outputs \citep{yoran2024makingretrievalaugmentedlanguagemodels}. Prior work like Blendfilter \citep{wang2024blendfilter} and RobustRAG \citep{xiang2024certifiablyrobustragretrieval} has attempted to address this issue through prompt-based filtering, but these approaches rely heavily on the model's intrinsic reasoning ability to identify irrelevant context, which is often insufficient for SLMs. In the meantime, some methods \citep{wei2025instructraginstructingretrievalaugmentedgeneration, yu-etal-2024-chain} investigate training LLMs to handle noisy contexts via supervised finetuning on data synthesized with RAG-retrieved content, but they significantly suffer from low robustness under noisy retrieval conditions.

Preference-based optimization methods \citep{zhang2024knowpoknowledgeawarepreferenceoptimization, dong2024understandllmneedsdual} have gained traction as alternatives to supervised fine-tuning by directly learning from relative preferences, offering improved robustness under noisy RAG settings. However, most approaches rely on single-pair comparisons, underutilizing the richer supervision available from multiple candidate responses, motivating recent work on multi-pair preference optimization. RoseRAG \citep{liu2025roserag} introduces margin-aware preference optimization by selecting the hardest positive and negative pair based on model likelihoods, improving robustness to noisy retrieval but discarding the remaining preference signals. Noise Contrastive Alignment (NCA) \citep{chen2024noisecontrastivealignmentlanguage} instead aggregates multiple negative responses within a contrastive framework, enhancing robustness but neglecting structured relationships among positive candidates. Other recent efforts \citep{xie2024mppomultipairwisepreference, gupta2025ampoactivemultipreferenceoptimization, wang2025singlerewardmultipairmultiperspectivepreference, le2024multireferencepreferenceoptimizationlarge} explore various strategies for leveraging multiple responses per query, including selective subset sampling, multi-perspective preference signals, and multiple reference distributions, yet they generally rely on hard selection or heuristic aggregation that limits full exploitation of structured multi-pair preference data.

In contrast, our framework jointly incorporates multiple preference pairs into a single optimization objective through a differentiable smoothing mechanism, preserving informative signals while maintaining computational efficiency.

\section{Preliminaries}
We briefly review the standard retrieval-augmented generation framework, preference optimization, and margin-aware preference selection, which form the basis of our method.

\subsection{Retrieval-Augmented Generation}
Given a language model parameterized by $\theta$, a knowledge base $\mathcal{K} = \{\mathcal{K}_i\}_{i=1}^{k}$ (where $k$ denotes the number of documents), a retriever $\mathcal{R}(\cdot)$, and a query $q$, the vanilla retrieval-augmented generation (RAG) framework follows a retrieve-then-generate paradigm. Specifically, it first retrieves the top-$K$ relevant documents from the knowledge base and then generates an answer conditioned on the retrieved information. This process can be formulated as
\begin{equation}
\begin{gathered}
\mathcal{K}_q = \mathcal{R}(q, \mathcal{K}; K), \\
y \sim P_\theta\bigl(y \mid \text{Prompt}_{\text{CoT}}(q, \mathcal{K}_q)\bigr)
\end{gathered}
\label{eq:rag}
\end{equation}
where $y$ denotes the generated response and $\text{Prompt}_{\text{CoT}}(\cdot)$ represents the CoT prompt.

\subsection{Preference Optimization}

Preference alignment for language models has traditionally relied on multi-stage procedures, such as reinforcement learning from human feedback (RLHF) \citep{ouyang2022training,bai2022training}, that require an additional reference model to guide and stabilize training. Recent advances have led to a class of methods that learn directly from preference data, avoiding or simplifying explicit reward modeling. These methods generally optimize a policy $\pi_\theta$ over preference pairs $(y_w, y_l)$ given a prompt $x$, where $y_w$ denotes a preferred response and $y_l$ a dispreferred one. A general form of such objectives can be expressed as

\begin{equation}
\mathcal{L}(\theta) = \mathcal{L}_{\text{base}}(\pi_\theta) - \lambda \mathbb{E}_{(x,y_w,y_l)}\left[\log \sigma\!\left(\beta \cdot h_\theta(y_w, y_l, x)\right) \right]
\label{eq:pref_general}
\end{equation}

where $\mathcal{L}_{\text{base}}$ is an optional base objective (e.g., supervised fine-tuning loss), $\sigma(\cdot)$ is the sigmoid function, $\beta$ controls the strength of the preference signal, $\lambda$ weights the preference term, and $h_\theta(y_w, y_l, x)$ is an algorithm-specific scoring function that quantifies the relative quality of $y_w$ over $y_l$. Different instantiations of $h_\theta$ and $\mathcal{L}_{\text{base}}$ give rise to distinct algorithms, including Direct Preference Optimization (DPO) \citep{rafailov2024directpreferenceoptimizationlanguage}, Contrastive Preference Optimization (CPO) \citep{xu2024contrastivepreferenceoptimizationpushing}, and Odds-Ratio Preference Optimization (ORPO) \citep{hong2024orpomonolithicpreferenceoptimization}, which differ in how reference normalization, base supervision, and contrast signals are incorporated.

\subsection{Margin-Aware Preference Optimization}
In RAG settings, preference optimization is particularly appealing because it learns from relative quality comparisons rather than treating any single response as ground truth, making it more robust to noisy retrieved context than supervised fine-tuning. When multiple preference pairs are available for a single query, selecting the most informative pair becomes critical for effective optimization. RoseRAG \citep{liu2025roserag} addresses this with a margin-aware selection strategy that identifies the hardest positive and negative examples based on current model likelihoods:
\begin{equation}
\begin{gathered}
    y_w^* = \arg\min_i \log P_\theta(y_w^i \mid x), \\
    y_l^* = \arg\max_j \log P_\theta(y_l^j \mid x),
\end{gathered}
\label{eq:roserag}
\end{equation}
That is, the preferred response with the lowest model confidence and the non-preferred response with the highest model confidence are selected, forming the most challenging contrast for the current policy. However, this hard argmin/argmax selection retains only a single pair per query, discarding the gradient signal from all remaining candidates.

\section{Method}

\label{sec:method}

We present {\name}, shown in Fig.~\ref{fig:framework}, a framework that improves preference optimization for SLM-based RAG by fully utilizing multi-pair preference data. {\name} consists of three components: (1) \textit{Multi-Pair Preference Data Generation}, which produces diverse chain-of-thought rationales via rejection sampling; (2) \textit{Multi-Pair Preference Data Aggregation}, which consolidates multiple preference pairs into a unified pseudo-pair through a differentiable smoothing operator; and (3) \textit{Preference Optimization}, which applies the smoothed objective to standard alignment algorithms.

% \begin{wrapfigure}{r}{0.6\columnwidth}
% \centering
%   \includegraphics[width=0.6\columnwidth]{figures/Framework.pdf}
%   \caption{Framework of Smoothing Preference Optimization with Multi-Pair Preference Data}
%   \label{fig:framework}
% \end{wrapfigure}

% \begin{figure}[h!]
% \centering
%   \includegraphics[width=1.0\columnwidth]{figures/Framework.pdf}
%   \caption {Framework of Smoothing Preference Optimization with Multi-Pair Preference Data}
%   \label{fig:framework}
% \end{figure}
\subsection{Multi-Pair Preference Data Generation}

% \begin{wrapfigure}{r}{0.6\columnwidth}
% \centering
%   \includegraphics[width=0.6\columnwidth]{figures/Framework.pdf}
%   \caption{Framework of Smoothing Preference Optimization with Multi-Pair Preference Data}
%   \label{fig:framework}
% \end{wrapfigure}

To construct preference data for training, we use the target SLM itself to generate chain-of-thought (CoT) rationales, and apply rejection sampling to filter for correctness. This avoids dependence on proprietary models and produces preference pairs that are well-calibrated to the target model's own capability.

We instantiate this procedure by constructing a multi-turn prompt consisting of a system message, a user message, and an initial assistant message. The system message specifies the task and instructs the model to reason step-by-step, acknowledging that the provided knowledge may contain irrelevant information. The user message provides the retrieved knowledge documents $\mathcal{K}_q$, the question $q$, and the ground-truth answer $a^*$. The assistant message initializes the reasoning process. The complete prompt template is provided in Appendix~\ref{sec:prompt}. Given these inputs, the model generates a rationale $r$ that captures its analytical reasoning.

To ensure quality, we apply rejection sampling~\citep{liu2024statisticalrejectionsamplingimproves, guo2025deepseek}. Given the generated rationale $r$, we extract an answer $a$ and compare it against the ground-truth label $a^*$. If $a$ matches $a^*$, the full response $(r, a)$ is retained as a preferred response $y_w$. The non-preferred response $y_l$ is obtained separately as the vanilla RAG output (Eqn.~\eqref{eq:rag}), which is generated without ground-truth guidance. By repeating this process with different sampling runs, we obtain multiple preferred--non-preferred pairs per query, yielding multi-pair preference data.

\begin{figure}[h!]
\centering
  \includegraphics[width=1.0\columnwidth]{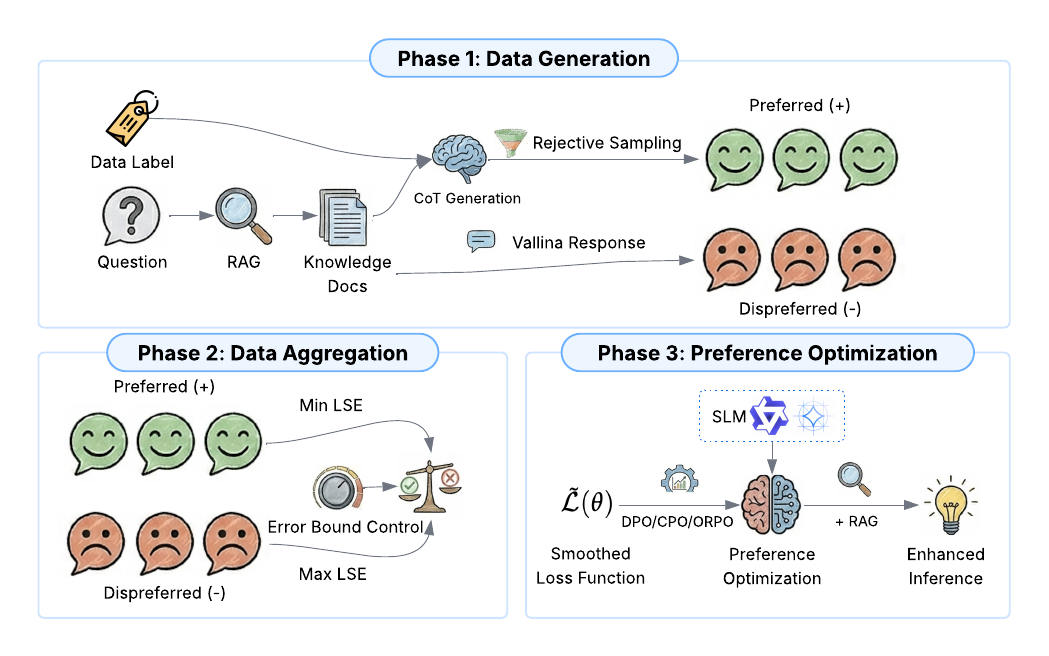}
  \caption {Framework of Smoothing Preference Optimization with Multi-Pair Preference Data}
  \label{fig:framework}
\end{figure}
\subsection{Multi-Pair Preference Data Aggregation}

The margin-aware strategy (Eqn.~\eqref{eq:roserag}) selects the single hardest preference pair via hard argmin/argmax, discarding all remaining candidates. Directly relaxing this selection poses competing challenges: the argmin/argmax operators are non-differentiable, preventing gradient signal from propagating through the selection step; yet naively replacing them with a simple average dilutes the focus on informative hard cases, weakening the discriminative structure that makes margin-aware optimization effective. Moreover, hard selection is performed once based on the initial model's likelihoods and remains fixed throughout training; as the policy evolves, the initially selected pair may no longer represent the most informative contrast. What is needed is a smooth, differentiable operator that preserves the emphasis on hard examples while incorporating signal from all candidates and adapting to the evolving policy.

We address this with the LogSumExp (LSE) function \citep{Nielsen_2016}, which provides a temperature-controlled smooth approximation to the maximum operator:
\begin{equation}
\text{LSE}_\tau(x_1, \cdots, x_n) = \frac{1}{\tau}\log \Big(\sum_{i=1}^n \exp({\tau x_i})\Big)
\label{eq:logsumexp}
\end{equation}
where $\tau > 0$ controls the sharpness of the approximation: as $\tau \to \infty$, the LSE operator converges to the exact maximum; as $\tau \to 0$, it reduces to a uniform average. This provides a principled interpolation between hard selection and soft aggregation. Because the LSE operator is evaluated using the current policy $\pi_\theta$ at each training step, the aggregation weights adapt dynamically as the model evolves, in contrast to the static one-time selection of RoseRAG. A formal analysis of the approximation bound is provided in Section~\ref{sec:theoretical} and Appendix~\ref{sec:logsumexp_err}.

Applying this operator to the policy likelihoods, we aggregate all preferred and non-preferred candidates into a unified pseudo-pair $(x, \tilde{y}_w, \tilde{y}_l)$. For the preferred set, we use the negated LSE to approximate the minimum-likelihood (hardest) positive; for the non-preferred set, we use LSE directly to approximate the maximum-likelihood (hardest) negative:
\begin{equation}
\begin{gathered}
    \log\pi_\theta(\tilde{y}_w|x) = -\text{LSE}_\tau(-\log\pi_\theta(y_w^1 \mid x), \cdots, -\log\pi_\theta(y_w^n \mid x)), \\
    \log\pi_\theta(\tilde{y}_l|x) = \text{LSE}_\tau(\log\pi_\theta(y_l^1 \mid x), \cdots, \log\pi_\theta(y_l^n \mid x)),
\end{gathered}
\label{eq:smooth_roserag}
\end{equation}

\subsection{Preference Optimization}
For a multi-pair preference data tuple $(x, \{y_w^i\}_{i=1}^n, \{y_l^i\}_{i=1}^n)$, where $x$ is the input prompt, $\{y_w^i\}$ is the set of preferred responses, and $\{y_l^i\}$ is the set of non-preferred responses, the smoothed preference optimization objective is:

  \begin{equation}
  \tilde{\mathcal{L}}(\theta) = \mathcal{L}_{\text{base}}(\pi_\theta) - \lambda \mathbb{E}_{(x,y_w,y_l)}\left[\log\sigma\left(\beta \cdot h_\theta(\tilde{y}_w, \tilde{y}_l, x)\right) \right]
  \label{eq:multi_pair_pref_loss}
  \end{equation}

where $\tilde{y}_w$ and $\tilde{y}_l$ are the smoothed pseudo-responses computed via Eqn.~\eqref{eq:smooth_roserag}. This formulation is algorithm-agnostic: for any preference algorithm instantiated through Eqn.~\eqref{eq:pref_general}, the smoothed variant is obtained by replacing the original policy likelihoods $\pi_\theta(y_w|x)$ and $\pi_\theta(y_l|x)$ within $h_\theta$ with their aggregated counterparts $\pi_\theta(\tilde{y}_w|x)$ and $\pi_\theta(\tilde{y}_l|x)$, directly yielding smoothed versions of DPO \citep{rafailov2024directpreferenceoptimizationlanguage}, CPO \citep{xu2024contrastivepreferenceoptimizationpushing}, and ORPO \citep{hong2024orpomonolithicpreferenceoptimization}.

\section{Theoretical Analysis}
\label{sec:theoretical}

\renewcommand{\P}{\mathbb{P}}

In this section, we provide theoretical justifications for the proposed smoothed aggregation mechanism. We first establish a formal error bound showing that the LSE approximation can be controlled via an interpretable hyperparameter (Section~\ref{sec:error}), and then analyze gradient stability to show that smooth aggregation yields provably lower variance than hard selection under candidate-pool sampling (Section~\ref{sec:gradient_alignment}).

\subsection{Approximation Error Analysis}
\label{sec:error}
The smoothed aggregation mechanism introduced in Section~4 relies on the LSE operator as a differentiable proxy for hard max/min selection. A natural question is how tightly this approximation tracks the exact maximum, and whether the approximation error can be controlled in a principled way. We establish a formal error bound and show how it can be used as a practical hyperparameter. Based on the LogSumExp smoothing in Eqn.~\eqref{eq:logsumexp}, we establish the following theorem to control the upper bound of the approximation error:
\begin{theorem}[LSE Error Bound Control]
\label{thm:lse_error_control}
Let $M>0$, $n>1$, and $x_1,\dots,x_n\in \mathbb{R}$. Let $x_{(1)}\le\cdots\le x_{(n)}$ and assume the top-gap
$\delta := x_{(n)}-x_{(n-1)} > 0$, $\epsilon:=\text{LSE}_\tau(x_1,\dots,x_n)-x_{(n)} ~$ for $\tau>0$. If $\tau \;\ge\; F_{\delta,n}(M)
\;:=\;\frac{1}{\delta}\,W\!\Big(\frac{\delta(n-1)}{M}\Big)$,
% \[
% \tau \;\ge\; F_{\delta,n}(M)
% \;:=\;\frac{1}{\delta}\,W\!\Big(\frac{\delta(n-1)}{M}\Big),
% \]
then $\epsilon \le M$. Moreover, for any $M>0$, $\delta>0$, and $n>1$, $F_{\delta,n}(M)\in\mathbb{R}$.
\end{theorem}

Given Theorem \ref{thm:lse_error_control}, we can treat the error bound $M$ as a hyperparameter in our smoothing strategy and adaptively control the temperature $\tau$ during training via $F_{\delta,n}(M)$, since the top-gap $\delta$ and smoothing size $n$ are known quantities. The proof is provided in Appendix \ref{sec:logsumexp_err}.

\subsection{Gradient Stability under Candidate-Pool Sampling}
\label{sec:gradient_alignment}

We now analyze stability by treating the sampled candidate pool itself as the source of randomness. For a fixed prompt $x$, one training step draws
\[
Y=(Y_w^1,\dots,Y_w^n,Y_l^1,\dots,Y_l^n)\sim P_x,
\]
and the CPO objective maps this pool to either a hard or a smoothed gradient signal. Hard selection keeps exactly one preferred/non-preferred pair, while smoothing aggregates the full pool through LSE. The key instability of hard selection is that the chosen pair itself is random: when different sampled pools select different hard pairs and these pairs induce different updates, this discrete selector randomness produces an intrinsic variance floor. By contrast, smoothing replaces this one-hot switch by a continuous weighted mixture whose fluctuation can be controlled through the softmax weighting structure.

Define
\[
V_{\mathrm{sm}}(x,\theta)
:=
\mathbb{E}_Y\!\Big[\big\|\nabla_\theta \mathcal{L}_{\mathrm{sm}}(\theta;x,Y)-\mathbb{E}_Y\nabla_\theta \mathcal{L}_{\mathrm{sm}}(\theta;x,Y)\big\|^2\Big],
\]
\[
V_{\mathrm{hard}}(x,\theta)
:=
\mathbb{E}_Y\!\Big[\big\|\nabla_\theta \mathcal{L}_{\mathrm{hard}}(\theta;x,Y)-\mathbb{E}_Y\nabla_\theta \mathcal{L}_{\mathrm{hard}}(\theta;x,Y)\big\|^2\Big].
\]
\begin{theorem}[Smooth Aggregation Improves Gradient Stability]
\label{thm:variance_cpo}
Under Appendix Assumptions~\ref{asm:grad_bound}, \ref{asm:grad_diameter}, and \ref{asm:selector_dominance},
\[
V_{\mathrm{hard}}(x,\theta)>V_{\mathrm{sm}}(x,\theta).
\]
\end{theorem}

Theorem~\ref{thm:variance_cpo} reduces the stability comparison to an interpretable tradeoff between two quantities. The term $p^2m^2$ is the selector-mixture variance floor induced by hard selection: it becomes large when the hard selector places non-trivial mass on multiple pair indices whose average updates are well separated. The right-hand side captures the fluctuation of smoothing: $\operatorname{Var}(\bar g(Y))$ is the branchwise averaging baseline, while the second term is the additional variance induced by soft weighting. Crucially, this extra term is multiplied by the shrinkage factor $\delta_n(\tau,R)$, which vanishes as the branch score range collapses. Hence, whenever hard selector competition is strong but the softmax weights remain close to a stable average, smooth aggregation yields strictly smaller gradient variance than hard selection. The full notation is given in Appendix~\ref{sec:appendix_notation}, and the supporting lemmas and proof appear in Appendix~\ref{sec:gradient_stability_proof}.
\section{Experiments}
\label{sec:experiments}

In this section, we evaluate the proposed {\name} by addressing the following research questions: 
\begin{enumerate}[leftmargin=3em]
    \item[RQ1]  How does {\name} perform compared to state-of-the-art baselines?
    \item[RQ2]  How does the proposed smoothed aggregation compare with alternative strategies?
    \item[RQ3]  How does {\name} perform across different preference optimization methods?
    \item[RQ4]  How does {\name} perform with varying numbers of retrieved documents?
    \item[RQ5]  How sensitive is {\name} to the hyperparameters error bound $M$ and smoothing size $n$?
\end{enumerate}
% \emph{RQ1)} How does {\name} perform compared to state-of-the-art baselines? \emph{RQ2)} How does the proposed smoothed aggregation compare with alternative strategies? \emph{RQ3)} How does {\name} perform across different preference optimization methods? \emph{RQ4)} How does {\name} perform with varying numbers of retrieved documents? \emph{RQ5)} How sensitive is {\name} to the hyperparameters error bound $M$ and smoothing size $n$?

\subsection{Datasets and Experiment Settings} 

\noindent\textbf{Datasets and Evaluation Metrics.} 
We conduct experiments on four public multi-hop question answering benchmarks: 
HotPotQA~(HPQA, \cite{yang-etal-2018-hotpotqa}), 
2WikiMultiHopQA~(2WMHQA, \cite{ho2020constructingmultihopqadataset}), 
MultiHopRAGQA~(MHRAGQA, \cite{tang2024multihoprag}), and
MuSiQueQA~(MSQQA, \cite{trivedi2022musiquemultihopquestionssinglehop}). 
Following prior work~\citep{shao2023enhancing,wang2024blendfilter,liu2025roserag}, 
we evaluate the first 500 questions from the development splits of HPQA and 2WMHQA.
For MSQQA, we select the first 500 questions from its development set.
As MHRAGQA lacks an official development split, we randomly set aside 20\% of the training data for evaluation and use the remaining 80\% for training,
with the final results also computed over 500 questions.
We employ exact match (EM) and F1 score as evaluation metrics, following \citet{liu2025roserag}.

\noindent\textbf{Baselines.} 
We compare our method against the following state-of-the-art baselines: 
1) CoT Prompting~\citep{wei2023chainofthoughtpromptingelicitsreasoning}, 
2) ReAct~\citep{yao2023reactsynergizingreasoningacting}, 
3) SelfAsk~\citep{press-etal-2023-measuring}, 
4) BlendFilter~\citep{wang2024blendfilter}, 
5) InstructRAG~\citep{wei2025instructraginstructingretrievalaugmentedgeneration},
6) RoseRAG~\citep{liu2025roserag}. 
Further background and details are deferred to Appendix \ref{sec:baseline}.

\noindent\textbf{Implementation Details.} 
Our experiments are conducted using two representative SLMs: Qwen2.5-1.5B-Instruct \citep{qwen2.5} and 
Gemma-2-2b-it \citep{team2024gemma}. 
For the retriever, we utilize the ColBERTv2 retriever~\citep{santhanam2022colbert}, a state-of-the-art efficient retrieval method implemented by \cite{khattab2023demonstrate}. 
The knowledge base we employ is the 2017 Wikipedia abstracts dump~\citep{khattab2023dspy}.
More implementation details can be found in Appendix \ref{sec:impl}.

\subsection{Overall Performance Comparison}

We evaluate the effectiveness of {\name} across four benchmarks against six strong baseline methods to answer RQ1.
Results for Qwen2.5-1.5B-Instruct and Gemma-2-2B-it are reported in Table \ref{tab:main_prfm_compare}.

{\name} consistently outperforms strong baselines across all benchmarks by a substantial margin.
The gains are especially pronounced on the more challenging multi-hop datasets (2WMHQA and MHRAGQA), where retrieved passages tend to be lengthy and noisy~\citep{liu2025roserag}, indicating that {\name} is effective at aggregating diverse evidence while maintaining reasoning consistency.
Prompt-based and decomposition-based methods (ReAct, SelfAsk, BlendFilter) are limited by SLMs’ intrinsic reasoning capacity, while fine-tuning-based approaches that explicitly enhance robustness to noisy contexts show stronger performance.
Overall, the superior and stable performance of {\name} highlights the importance of designing optimization strategies specifically for SLMs rather than directly transferring techniques developed for large-scale models.
\begin{table*}[t]
\centering
\caption{Performance comparison of different methods with Qwen2.5-1.5B-Instruct and gemma-2-2b-it as the backbone models. 
The best and second-best results are highlighted in \textbf{bold} and \underline{underlined}, respectively.
(HPQA: HotPotQA, 2WMHQA: 2WikiMultiHopQA, MHRAGQA: MultiHopRAGQA, MSQQA: MuSiQueQA)}
\label{tab:main_prfm_compare}
\small
\setlength{\tabcolsep}{2pt}
\resizebox{0.95\textwidth}{!}{%
\begin{tabular}{
>{\bfseries}l
>{\columncolor{lightred}}c
>{\columncolor{lightred}}c
>{\columncolor{lightpurple}}c
>{\columncolor{lightpurple}}c
>{\columncolor{lightorange}}c
>{\columncolor{lightorange}}c
>{\columncolor{lightblue}}c
>{\columncolor{lightblue}}c
% |
>{\columncolor{lightred}}c
>{\columncolor{lightred}}c
>{\columncolor{lightpurple}}c
>{\columncolor{lightpurple}}c
>{\columncolor{lightorange}}c
>{\columncolor{lightorange}}c
>{\columncolor{lightblue}}c
>{\columncolor{lightblue}}c
}
\toprule[0.3ex]
\multirow{3}{*}{\bf Method}
& \multicolumn{8}{c}{\bf Qwen2.5-1.5B-Instruct}
& \multicolumn{8}{c}{\bf Gemma-2-2b-it} \\
\cmidrule(lr){2-9} \cmidrule(lr){10-17}
& \multicolumn{2}{>{\columncolor{lightred}}c}{\bf HPQA}
& \multicolumn{2}{>{\columncolor{lightpurple}}c}{\bf 2WMHQA}
& \multicolumn{2}{>{\columncolor{lightorange}}c}{\bf MHRAGQA}
& \multicolumn{2}{>{\columncolor{lightblue}}c}{\bf MSQQA}
& \multicolumn{2}{>{\columncolor{lightred}}c}{\bf HPQA}
& \multicolumn{2}{>{\columncolor{lightpurple}}c}{\bf 2WMHQA}
& \multicolumn{2}{>{\columncolor{lightorange}}c}{\bf MHRAGQA}
& \multicolumn{2}{>{\columncolor{lightblue}}c}{\bf MSQQA} \\
\cmidrule(lr){2-3} \cmidrule(lr){4-5} \cmidrule(lr){6-7} \cmidrule(lr){8-9}
\cmidrule(lr){10-11} \cmidrule(lr){12-13} \cmidrule(lr){14-15} \cmidrule(lr){16-17}
& EM & F1 & EM & F1 & EM & F1 & EM & F1
& EM & F1 & EM & F1 & EM & F1 & EM & F1 \\
\midrule
CoT
& 26.6 & 38.1 & 17.4 & 22.9 & 3.8 & 4.8 & 3.0 & 9.6
& 37.2 & 48.4 & 21.6 & 26.8 & 19.8 & 20.6 & 6.8 & 15.1 \\
vanilla RAG
& 29.0 & 40.6 & 15.2 & 23.32 & 7.4 & 8.39 & 7.6 & 15.7
& 39.2 & 50.2 & 19.4 & 24.9 & 15.6 & 16.1 & 9.6 & 17.3 \\
ReAct
& 12.0 & 22.7 & 6.6 & 20.9 & 2.8 & 3.6 & 3.6 & 10.5
& 34.8 & 47.3 & 26.4 & 33.6 & 23.8 & 24.6 & \textbf{13.0} & \underline{23.1} \\
SelfAsk
& 22.2 & 33.9 & 18.4 & 24.2 & 15.0 & 16.1 & 3.0 & 11.4
& 32.0 & 42.4 & 28.4 & 34.3 & 29.8 & 30.2 & 9.6 & 16.8 \\
BlendFilter
& 25.0 & 33.9 & 12.6 & 19.1 & 4.0 & 5.0 & 3.2 & 9.9
& 39.8 & 50.6 & 23.0 & 28.9 & 24.4 & 26.1 & 9.4 & 17.1 \\
InstructRAG
& 15.2 & 23.5 & 17.8 & 23.0 & 16.0 & 16.2 & 2.2 & 7.8
& 28.4 & 37.2 & 16.0 & 20.4 & 15.2 & 15.7 & 3.2 & 7.0 \\
RoseRAG
& \underline{36.6} & \underline{46.0} & \underline{29.6} & \underline{33.7} & \underline{45.4} & \underline{46.3} & \underline{9.4} & \underline{18.6}
& \underline{41.4} & \underline{52.6} & \underline{32.4} & \underline{36.9} & \underline{58.2} & \underline{58.5} & \underline{12.4} & 21.6 \\
\midrule
\textbf{\name}
& \textbf{39.0} & \textbf{49.3} & \textbf{32.2} & \textbf{37.7} & \textbf{52.6} & \textbf{53.0} & \textbf{10.8} & \textbf{18.6}
& \textbf{42.6} & \textbf{53.0} & \textbf{35.8} & \textbf{40.0} & \textbf{59.0} & \textbf{59.2} & \underline{12.4} & \textbf{23.4} \\
\bottomrule[0.3ex]
\end{tabular}}
\end{table*}

\subsection{Effectiveness of Data Aggregation Methods}
\begin{table}[t]
\centering
\caption {
Performance (EM/F1) comparison of {\name} with different aggregation strategies. 
Experiments are conducted on the MHRAGQA using Qwen2.5-1.5B-Instruct as the backbone model. (\textit{Weighted}: weight training data based on untrained model's likelihood;
\textit{Average}: assign equal weights. Detailed explanations of these strategies are provided in Appendix Table~\ref{tab:agg_strategies}.}
\label{tab:agg_ablation}
% \small
\setlength{\tabcolsep}{4pt}
\begin{tabular}{cc cc cc}
\toprule[0.3ex]
\multicolumn{2}{c}{\textbf{Aggregation Strategy}} & \multicolumn{2}{c}{\textbf{Smoothing size = 3}} & \multicolumn{2}{c}{\textbf{Smoothing size = 4}} \\
\midrule
\textbf{Positive} & \textbf{Negative} & \textbf{EM} & \textbf{F1} & \textbf{EM} & \textbf{F1} \\
\cmidrule{1-2}\cmidrule(l){3-4}\cmidrule(l){5-6}
Average & Average   & 48.0 & 48.5 & 50.6 & 51.2 \\
Average & Weighted  & 47.0 & 47.4 & 43.4 & 44.3 \\
Weighted & Average  & 45.2 & 46.0 & 48.8 & 49.1 \\
Weighted & Weighted & 48.2 & 48.7 & 47.4 & 47.8 \\
\midrule
\multicolumn{2}{c}{\name-DPO} & \textbf{50.80} & \textbf{51.43} & \textbf{52.40} & \textbf{52.98} \\
\bottomrule[0.3ex]
\end{tabular}
\end{table}

To answer RQ2, we compare our proposed smoothed aggregation with two alternative strategies for combining multi-pair preference data $(x, y_w^1, \cdots, y_w^n, y_l^1, \cdots, y_l^n)$ into a pseudo-pair:
\emph{Weighted} (fixed weights based on base model likelihoods),
\emph{Average} (equal weights).
More details can be found in Appendix \ref{sec:agg_method}.

Results are reported in Table~\ref{tab:agg_ablation}, using Qwen2.5-1.5B-Instruct as the backbone.
Across two smoothing sizes, all heuristic-based methods exhibit significantly worse performance than our proposed {\name}.
Average and Weighted aggregation dilute informative signals by mixing responses of heterogeneous quality.
In contrast, our proposed {\name} consistently achieves the best performance and remains stable across different smoothing sizes, confirming that LSE smoothing effectively balances the focus on hard examples with the utilization of all candidates.

\subsection{Effectiveness Across Different Preference Optimization Methods}

\begin{figure*}[t]
\includegraphics[width=\columnwidth]{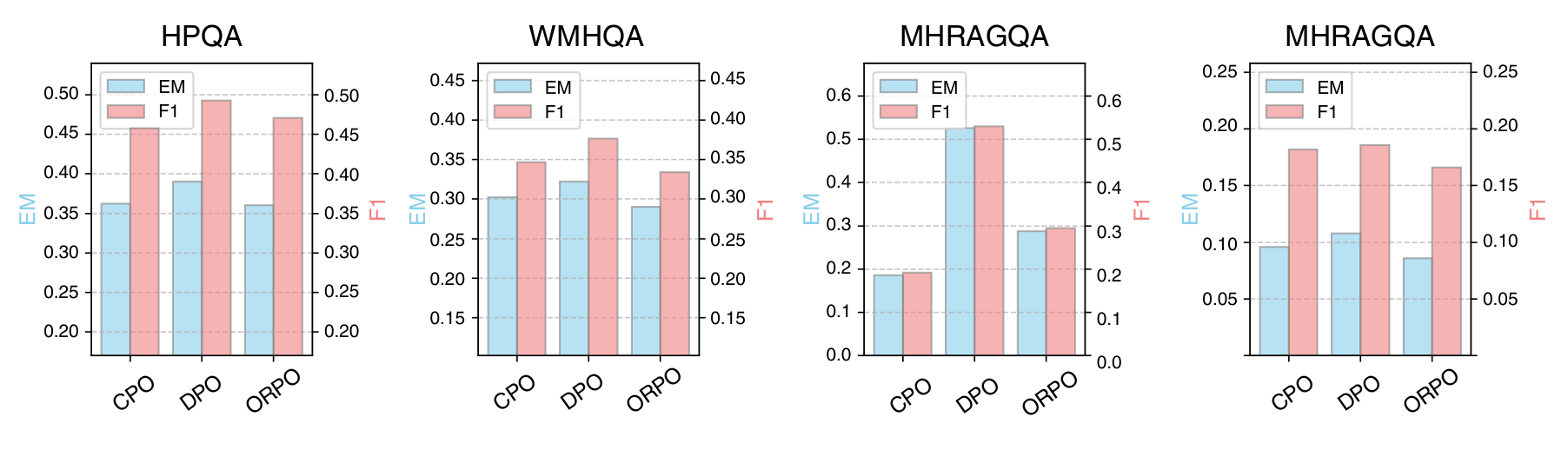}
\caption {
Performance (EM/F1) of {\name} when using different preference optimization methods.
The backbone model is Qwen2.5-1.5B-Instruct.
}
\label{fig:result_algo_comparison}
\end{figure*}

Figure~\ref{fig:result_algo_comparison} reports the performance of three preference optimization methods (DPO, CPO, ORPO) with a fixed smoothing component to answer RQ3.
{\name} consistently achieves strong and stable performance across all three methods, demonstrating the robustness of the proposed smoothing mechanism.
Among them, DPO yields the highest EM and F1 scores, with the gap being especially pronounced on more challenging multi-hop datasets such as MHRAGQA, suggesting that DPO's reference-policy regularization is particularly beneficial for capturing complex reasoning chains under noisy retrieval.

\subsection{Effectiveness with a Varying Number of Retrieved Documents}

Figure~\ref{fig:rag_count} compares {\name} and vanilla RAG on HPQA, with the number of retrieved documents $K$ varying from 0 to 10 to answer RQ4.
When $K=0$, the backbone relies solely on its internal knowledge; {\name} demonstrates comparable performance to vanilla RAG, indicating that the optimization did not degrade the model's intrinsic reasoning abilities and pretrained knowledge.

As $K$ increases from 1 to 5, both methods benefit from improved retrieval recall, but {\name} exhibits much steadier improvement while vanilla RAG shows only marginal gains beyond $K=1$.
When $K > 5$, vanilla RAG performance begins to degrade due to increased noise, whereas {\name} maintains its superior performance and shows greater resilience.
These results demonstrate that {\name} provides more robust evidence utilization than vanilla RAG, enabling consistent performance gains even as the retrieved context grows noisier.

\begin{figure*}[htbp]
\centering
\begin{subfigure}[t]{0.3\textwidth}
\centering
\includegraphics[height=3.6cm]{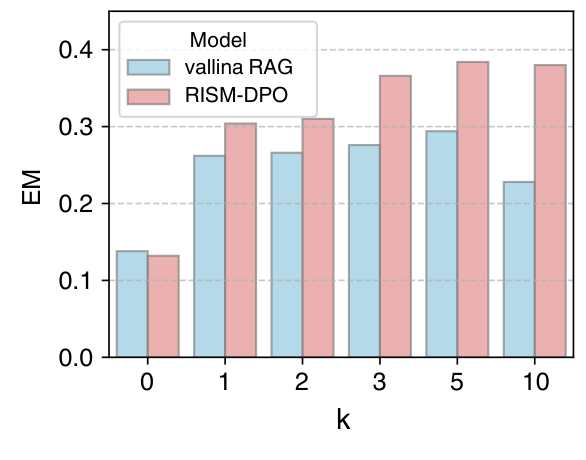}
\caption{}
\label{fig:rag_count}
\end{subfigure}\hfill
\begin{subfigure}[t]{0.35\textwidth}
\centering
\includegraphics[height=3.6cm]{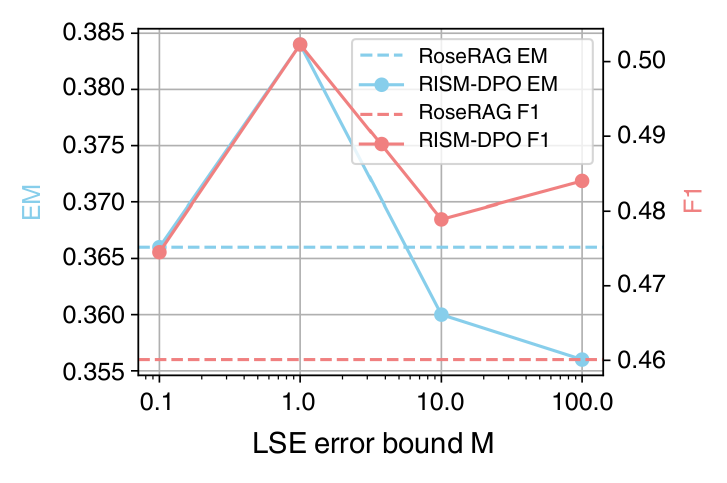}
\caption{}
\label{fig:error_bound}
\end{subfigure}\hfill
\begin{subfigure}[t]{0.3\textwidth}
\centering
\includegraphics[height=3.6cm]{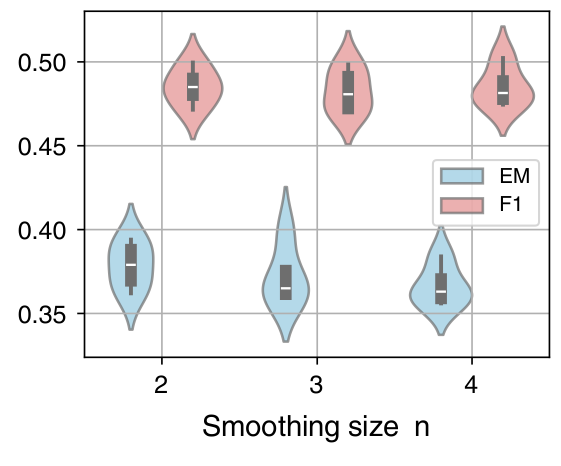}
\caption{}
\label{fig:smooth_size}
\end{subfigure}
\caption{Sensitivity analysis of {\name} across different hyperparameter settings. Experiments are conducted on the HPQA dataset using Qwen2.5-1.5B-Instruct as the backbone model. (a) EM under varying retrieved document counts. (b) EM/F1 under varying LSE approximation error bound $M$ with smoothing size $n=4$. (c) EM/F1 under varying smoothing size $n$; inner box shows the interquartile range across 4 replicates and outer violin shows the density.}
\label{fig:hyperparam}
\end{figure*}

\subsection{Sensitivity Analysis of Hyperparameters}

\textbf{LSE Approximation Error Analysis.}
The error bound $M$ in Theorem~\ref{thm:lse_error_control} controls how closely the LSE operator approximates the exact max/min, with the temperature $\tau$ adaptively adjusted via $F_{\delta,n}(M)$ during training.
Figure~\ref{fig:error_bound} shows the impact of $M$ on model performance.
When $M$ is small, {\name} achieves performance comparable to RoseRAG, as the LSE operator degenerates to hard selection, validating our theoretical analysis.
Performance peaks at $M=1$, suggesting that a moderately relaxed bound enables the model to leverage a richer set of preference pairs while preserving the discriminative structure of margin-aware selection.

\textbf{Smoothing Size.}
The smoothing size $n$ controls the number of candidate responses aggregated in the LSE operator.
Figure~\ref{fig:smooth_size} shows the performance under different smoothing sizes with a fixed error bound $M$, where $\tau$ is adaptively adjusted following Theorem~\ref{thm:lse_error_control}.
Both EM and F1 scores exhibit a clear degradation as $n$ increases, even when the worst-case approximation error is controlled.
This indicates that larger smoothing sizes dilute the influence of top-ranked candidates by aggregating over a broader spectrum of scores; while controlling $M$ guarantees closeness to the true max in magnitude, increasing $n$ alters the relative contributions of subdominant candidates.
Consequently, moderate smoothing sizes achieve a better balance between optimization smoothness and ranking fidelity.

\section{Conclusions}
\label{sec:conclusion}
% We presented {\name}, a preference optimization framework that addresses the underutilization of multi-pair preference data in SLM-based RAG. Rather than selecting only the single hardest preference pair via hard argmin/argmax, {\name} introduces a differentiable LSE-based smoothed aggregation that consolidates all preferred and non-preferred candidates into a unified pseudo-pair. This soft aggregation preserves gradient signal from every candidate while retaining the discriminative structure of margin-aware selection, and dynamically adapts the aggregation weights as the policy evolves during training. The formulation is algorithm-agnostic, directly extending DPO, CPO, and ORPO without modification to the underlying optimization procedures. We further provided a theoretical analysis bounding the approximation error introduced by LSE smoothing, translating the error control into an explicit and interpretable hyperparameter that enables principled trade-off between hard selection and smooth aggregation. Extensive experiments on four multi-hop question answering benchmarks with two SLM backbones demonstrate that {\name} consistently outperforms state-of-the-art baselines in both Exact Match and F1, confirming that structured multi-pair aggregation is critical for robust evidence utilization under noisy retrieval conditions.
We presented {\name}, a preference optimization framework that addresses the underutilization of multi-pair preference data in SLM-based RAG. The core idea is to replace hard pair selection with a differentiable LSE-based smoothing mechanism that aggregates all preference candidates into a unified objective, preserving informative gradient signals while retaining the discriminative structure of margin-aware optimization. The resulting formulation is algorithm-agnostic, directly applicable to DPO, CPO, and ORPO without modification. We further provided theoretical guarantees showing that the LSE approximation error is controllable via an interpretable hyperparameter, and that smooth aggregation yields provably tighter gradient alignment than hard selection. Extensive experiments on four multi-hop question answering benchmarks with two representative SLM backbones demonstrate that {\name} consistently outperforms state-of-the-art baselines in both Exact Match and F1, highlighting the importance of structured multi-pair aggregation for robust evidence utilization under noisy retrieval conditions.

\bibliography{colm2026_conference}
\bibliographystyle{colm2026_conference}

\appendix

% ---------- Preliminaries ----------
\newcommand{\LSE}{\operatorname{LSE}}
\newcommand{\R}{\mathbb{R}}

\section{Omitted Proofs}
\label{sec:omitted_proofs}

This section collects all proofs omitted from the main text. Appendix~\ref{sec:logsumexp_err} proves the approximation-error result used in the main text. Appendix~\ref{sec:appendix_notation} then switches to the CPO setting and analyzes gradient stability under candidate-pool sampling. The two analyses are independent: the first controls LSE approximation error, while the second explains why smoothing produces a more stable training signal than hard pair selection.

%----------------------------------------------------------------------
\subsection{LogSumExp Error Bound Analysis}
\label{sec:logsumexp_err}
%----------------------------------------------------------------------

\begin{lemma}[LSE Error Bound]
\label{lem:lse_gap_bound}
Let $n>1$, $\tau>0$, and $x_1,\dots,x_n\in\R$. Let $x_{(1)}\le \cdots \le x_{(n)}$ be the order statistics and define
\[
\LSE_\tau(x_1,\dots,x_n)\;:=\;\frac1\tau\log\Big(\sum_{i=1}^n e^{\tau x_i}\Big).
\]
If $\delta:=x_{(n)}-x_{(n-1)}\ge 0$, then the error
$\epsilon \;:=\;\LSE_\tau(x_1,\dots,x_n)-x_{(n)}$
satisfies
\[
\epsilon \le \frac{1}{\tau} \log n 
\text{ and } 
\epsilon \le \frac{n-1}{\tau}e^{-\tau\delta} 
\]
\end{lemma}

\begin{proof}
Factor out $e^{\tau x_{(n)}}$ from the exponential sum:
\begin{equation*}
\begin{split}
\sum_{i=1}^n e^{\tau x_i}
&= e^{\tau x_{(n)}}\sum_{i=1}^n e^{\tau(x_{(i)}-x_{(n)})}\\
&= e^{\tau x_{(n)}}\Big(1+\sum_{i=1}^{n-1} e^{\tau(x_{(i)}-x_{(n)})}\Big)
\end{split}
\end{equation*}
Therefore
\[
\epsilon
= \frac1\tau\log\Big(1+\sum_{i=1}^{n-1} e^{\tau(x_{(i)}-x_{(n)})}\Big)
\]

\noindent\textbf{Uniform bound.}
Since $x_{(i)}\le x_{(n)}$ for all $i$, we have $e^{\tau(x_{(i)}-x_{(n)})}\le 1$, and thus
\[
\sum_{i=1}^n e^{\tau(x_{(i)}-x_{(n)})}\le n.
\]
Consequently,
\[
\epsilon
=\frac1\tau \log\!\Big(\sum_{i=1}^n e^{\tau(x_{(i)}-x_{(n)})}\Big)
\le \frac1\tau \log n.
\]

\noindent\textbf{$\delta$-Bound.} For each $i\le n-1$, we have $x_{(i)}\le x_{(n-1)}$, hence $
x_{(i)}-x_{(n)} \le x_{(n-1)}-x_{(n)} = -\delta,
$, so $e^{\tau(x_{(i)}-x_{(n)})}\le e^{-\tau\delta}$.
Summing these bounds gives
\[
\sum_{i=1}^{n-1} e^{\tau(x_{(i)}-x_{(n)})} \le (n-1)e^{-\tau\delta},
\]
and thus
\[
\epsilon \le \frac1\tau\log\Big(1+(n-1)e^{-\tau\delta}\Big).
\]
Finally, since $\log(1+u)\le u$ for all $u\ge 0$, we obtain
\[
\epsilon \le \frac1\tau\log\Big(1+(n-1)e^{-\tau\delta}\Big)
\le \frac{n-1}{\tau}e^{-\tau\delta}.
\]
\end{proof}

\begin{lemma}[Lambert $W$ inversion equivalence]\label{lem:lambert_inversion}
Let $a>0$. For any $y>0$,
\[
ye^y \ge a \;\Longleftrightarrow\; y \ge W(a),
\]
where $W$ is the (principal) Lambert $W$ function and $W(a) \in \mathbb{R}$.
\end{lemma}

\begin{proof}
Fix $a>0$ and consider the function
\[
f:(0,\infty)\to(0,\infty),\quad f(y)=y e^{y}.
\]
We first show that $f$ is strictly increasing on $(0,\infty)$. Indeed,
\[
f'(y)=e^{y}(1+y)>0 \qquad \text{for all } y>0,
\]
so $f$ is strictly increasing and hence injective.

By definition of the (principal) Lambert $W$ function, $W(a)$ is the unique real number $w$ such that
\[
w e^{w}=a.
\]
Existence and uniqueness follow here because $f$ is continuous, strictly increasing, and satisfies
$\lim_{y\to 0} f(y)=0$ and $\lim_{y\to\infty} f(y)=\infty$; thus for each $a>0$ there is a unique $w>0$ with $f(w)=a$.

Since $f$ is strictly increasing and $f(W(a))=a$, for any $y>0$ we have
\[
ye^y \ge a = f(W(a)) \;\Longleftrightarrow\; f(y)\ge f(W(a)) \;\Longleftrightarrow\; y\ge W(a). \qedhere
\]
\end{proof}

\begin{theorem}[LSE Error Bound Control]
Let $M>0$, $n>1$, and $x_1,\dots,x_n\in\R$. Let $x_{(1)}\le\cdots\le x_{(n)}$ and assume the top-gap
$\delta := x_{(n)}-x_{(n-1)} > 0$, $\epsilon:=\LSE_\tau(x_1,\dots,x_n)-x_{(n)} ~$ for $\tau>0$. If
\[
\tau \;\ge\; F_{\delta,n}(M)
\;:=\;\frac{1}{\delta}\,W\!\Big(\frac{\delta(n-1)}{M}\Big),
\]
then $\epsilon \le M$. Moreover, for any $M>0$, $\delta>0$, and $n>1$, $F_{\delta,n}(M)\in\R$.
\end{theorem}

\begin{proof}
By Lemma~\ref{lem:lse_gap_bound},
\[
\epsilon \le \frac{n-1}{\tau}e^{-\tau\delta}.
\]
Thus it suffices to show that $\tau \ge F_{\delta,n}(M)$ implies $\frac{n-1}{\tau}e^{-\tau\delta} \le M$.
Let $y:=\tau\delta>0$ and $a:=\frac{\delta(n-1)}{M}>0$. Since $\tau \ge F_{\delta,n}(M)$ means $y \ge W(a)$,
Lemma~\ref{lem:lambert_inversion} (applied right-to-left) gives $ye^y \ge a$, i.e.,
\[
\frac{\delta(n-1)}{M} \le ye^y
\;\Longleftrightarrow\;
\frac{\delta(n-1)}{M}\,e^{-y}\le y
\;\Longleftrightarrow\;
\frac{n-1}{\tau}e^{-\tau\delta} \le M.
\]
Hence $\epsilon \le \frac{n-1}{\tau}e^{-\tau\delta} \le M$.

For the final claim, since $M>0$, $\delta>0$, and $n>1$, the argument $\displaystyle\frac{\delta(n-1)}{M}$ is strictly positive.
The principal branch $W(z)$ is real-valued for all $z\ge 0$, hence
$\displaystyle W\!\big(\frac{\delta(n-1)}{M}\big)\in\R$ and therefore $F_{\delta,n}(M)\in\R$.
\end{proof}

\subsection{Notation for Gradient Stability Analysis}
\label{sec:appendix_notation}

Throughout this appendix, $\|\cdot\|$ denotes the Euclidean norm.
For probability vectors $p,q\in\mathbb{R}^n$, we use the total variation distance
\[
\operatorname{TV}(p,q):=\frac12\sum_{i=1}^n |p_i-q_i|
=
1-\sum_{i=1}^n \min\{p_i,q_i\}.
\]

\paragraph{CPO-Specific Quantities.}

Let
\[
\phi(u):=-\log \sigma(u).
\]
For a fixed prompt $x$ and sampled candidate pool
\[
Y=(Y_w^1,\dots,Y_w^n,Y_l^1,\dots,Y_l^n)\sim P_x,
\]
define branchwise scores and score gradients by
\[
s_i^w(Y):=\log \pi_\theta(Y_w^i\mid x),
\qquad
s_i^l(Y):=\log \pi_\theta(Y_l^i\mid x),
\]
\[
u_i(Y):=\nabla_\theta \log \pi_\theta(Y_w^i\mid x),
\qquad
v_i(Y):=\nabla_\theta \log \pi_\theta(Y_l^i\mid x).
\]
The hard-selected indices and smoothed branch weights are
\[
i_w^\star(Y):=\arg\min_i s_i^w(Y),
\qquad
i_l^\star(Y):=\arg\max_i s_i^l(Y),
\qquad
I(Y):=(i_w^\star(Y),i_l^\star(Y)),
\]
\[
q_{w,i}(Y):=\frac{e^{-\tau s_i^w(Y)}}{\sum_{j=1}^n e^{-\tau s_j^w(Y)}},
\qquad
q_{l,i}(Y):=\frac{e^{\tau s_i^l(Y)}}{\sum_{j=1}^n e^{\tau s_j^l(Y)}}.
\]
For hard selection, we interpret the update on each sampled pool $Y$ as the active-branch gradient, i.e., the gradient evaluated with the selected indices $i_w^\star(Y)$ and $i_l^\star(Y)$ held fixed. Equivalently, away from selector-switching boundaries, the hard objective is locally smooth and the formula below is its ordinary gradient.
Define the hard and smoothed margins by
\[
\Delta_{\mathrm{hard}}(Y):=
\beta\bigl(s_{i_w^\star(Y)}^w(Y)-s_{i_l^\star(Y)}^l(Y)\bigr),
\]
\[
\Delta_{\mathrm{sm}}(Y):=
\beta\Big(-\LSE_\tau(-s_1^w(Y),\dots,-s_n^w(Y))
-\LSE_\tau(s_1^l(Y),\dots,s_n^l(Y))\Big).
\]
The corresponding gradients are
\begin{align}
g_{\mathrm{hard}}(Y)
&:=\beta\,\phi'(\Delta_{\mathrm{hard}}(Y))
\bigl(u_{i_w^\star(Y)}(Y)-v_{i_l^\star(Y)}(Y)\bigr), \label{eq:hard_grad_exact}\\
g_{\mathrm{sm}}(Y)
&:=\beta\,\phi'(\Delta_{\mathrm{sm}}(Y))
\left(\sum_{i=1}^n q_{w,i}(Y)u_i(Y)-\sum_{i=1}^n q_{l,i}(Y)v_i(Y)\right). \label{eq:sm_grad_exact}
\end{align}

We also use
\[
\bar u(Y):=\frac1n\sum_{i=1}^n u_i(Y),
\qquad
\bar v(Y):=\frac1n\sum_{i=1}^n v_i(Y),
\]
and the corresponding averaged smoothed gradient
\[
\bar g(Y):=\beta\,\phi'(\Delta_{\mathrm{sm}}(Y))\bigl(\bar u(Y)-\bar v(Y)\bigr).
\]

For any square-integrable vector random variable $X$, we use
\[
\operatorname{Var}(X):=\mathbb{E}\!\left[\|X-\mathbb{E}X\|^2\right].
\]
For probability vectors $p,q\in\mathbb{R}^n$, we use the total variation distance
\[
\operatorname{TV}(p,q):=\frac12\sum_{i=1}^n |p_i-q_i|
=
1-\sum_{i=1}^n \min\{p_i,q_i\}.
\]
We write
\[
\mathrm{Unif}_n:=\left(\frac1n,\dots,\frac1n\right)
\]
for the uniform distribution on $n$ points.

For the hard lower bound, define for each pair index $i$
\[
p_i:=\P(I(Y)=i),
\qquad
\mu_i:=\mathbb{E}[g_{\mathrm{hard}}(Y)\mid I(Y)=i].
\]
We also use the two-index specialization: there exist $a\neq b$ and constants $p\in(0,1/2]$, $m>0$ such that
\[
\P(I(Y)=a)\ge p,
\qquad
\P(I(Y)=b)\ge p,
\qquad
\|\mu_a-\mu_b\|\ge m.
\]

Finally, define
\[
\delta_n(\tau,R):=
\frac{(n-1)(e^{\tau R}-1)}{n(e^{\tau R}+n-1)}.
\]

\paragraph{Assumptions for the Stability Comparison.}

\begin{assumption}[Bounded branch scores and score gradients]
\label{asm:grad_bound}
\label{asm:score_range}
There exist constants $G>0$ and $R_w,R_l\ge 0$ such that
\[
\|u_i(Y)\|\le G,
\qquad
\|v_i(Y)\|\le G
\]
for every candidate index $i$ and every sampled pool $Y$, and
\[
\max_i s_i^w(Y)-\min_i s_i^w(Y)\le R_w,
\qquad
\max_i s_i^l(Y)-\min_i s_i^l(Y)\le R_l
\]
for every sampled pool $Y$.
\end{assumption}

\begin{assumption}[Bounded within-pool gradient diameters]
\label{asm:grad_diameter}
There exist $D_w,D_l\ge 0$ such that
\[
\max_{i,j}\|u_i(Y)-u_j(Y)\|\le D_w,
\qquad
\max_{i,j}\|v_i(Y)-v_j(Y)\|\le D_l
\]
for every sampled pool $Y$.
\end{assumption}

\begin{assumption}[Selector-mixture dominance]
\label{asm:selector_dominance}
The selector-mixture variance floor of hard selection dominates the fluctuation budget of smoothing:
\[
p^2m^2
>
2\operatorname{Var}(\bar g(Y))
+
2\beta^2\Big(D_w\,\delta_n(\tau,R_w)+D_l\,\delta_n(\tau,R_l)\Big)^2.
\]
\end{assumption}

Assumption~\ref{asm:selector_dominance} is the explicit comparison condition used in the main theorem. Its left-hand side is the irreducible selector-mixture variance floor induced by hard selection, while the right-hand side measures the total fluctuation budget of smoothing: the branchwise averaging baseline $\operatorname{Var}(\bar g(Y))$ plus the additional soft-weighting term controlled by the diameters $D_w,D_l$ and the shrinkage factor $\delta_n(\tau,R)$. Intuitively, the condition is most plausible when hard selector competition is strong but the softmax weights remain close to a stable branchwise average.

%----------------------------------------------------------------------
\subsection{Auxiliary Lemmas for Gradient Stability}
\label{sec:gradient_stability_proof}
%----------------------------------------------------------------------

\begin{lemma}[Hard selector induces a variance floor]
\label{lem:hard_lower}
Under the notation above,
\[
V_{\mathrm{hard}}(x,\theta)
\ge
\frac12\sum_{i,j} p_i p_j\|\mu_i-\mu_j\|^2.
\]
In particular, if there exist $a\neq b$ and $p\in(0,1/2]$ such that
\[
\P(I(Y)=a)\ge p,
\qquad
\P(I(Y)=b)\ge p,
\qquad
\|\mu_a-\mu_b\|\ge m,
\]
then
\[
V_{\mathrm{hard}}(x,\theta)\ge p^2m^2.
\]
\end{lemma}

\begin{proof}
Let
\[
Z(Y):=g_{\mathrm{hard}}(Y),
\]
Applying the law of total variance with respect to the hard-selected pair index
\[
\operatorname{Var}(Z)
=
\mathbb{E}\!\bigl[\operatorname{Var}(Z\mid I(Y))\bigr]
+ \operatorname{Var}\!\bigl(\mathbb{E}[Z\mid I(Y)]\bigr),
\]
and therefore
\[
V_{\mathrm{hard}}(x,\theta)=\operatorname{Var}(Z)\ge \operatorname{Var}\!\bigl(\mathbb{E}[Z\mid I(Y)]\bigr).
\]
Write
\[
\mu_j:=\mathbb{E}[Z\mid I(Y)=j],
\qquad
p_j:=\P(I(Y)=j).
\]
Then
\[
\operatorname{Var}\!\bigl(\mathbb{E}[Z\mid I(Y)]\bigr)
=
\sum_j p_j\left\|\mu_j-\sum_k p_k\mu_k\right\|^2
=
\frac12\sum_{i,j} p_i p_j\|\mu_i-\mu_j\|^2.
\]
This proves the first claim.

For the second claim, suppose there exist two indices $a\neq b$ such that $p_a\ge p$, $p_b\ge p$, and
\[
\|\mu_a-\mu_b\|
\ge m.
\]
Keeping only the $(a,b)$ contribution in the pairwise-variance formula, we obtain
\[
\operatorname{Var}\!\bigl(\mathbb{E}[Z\mid I(Y)]\bigr)
\ge
p_a p_b \|\mu_a-\mu_b\|^2
\ge
p^2m^2.
\]
Hence
\[
V_{\mathrm{hard}}(x,\theta)\ge p^2m^2.
\]
\end{proof}

\begin{remark}[Interpretation]
Lemma~\ref{lem:hard_lower} isolates the instability mechanism specific to hard selection without conditioning on any auxiliary event. Whenever the hard selector places non-trivial probability mass on multiple pair indices and those indices induce distinct average updates, the resulting selector mixture contributes a non-zero variance floor. Smoothing removes this discrete mixture by replacing one-hot selection with a continuous weighted aggregation, which is why it is more stable.
\end{remark}

%----------------------------------------------------------------------
\begin{lemma}[Refined upper bound for smoothed variance]
\label{lem:smooth_upper}
Under Assumptions~\ref{asm:grad_bound} and \ref{asm:grad_diameter},
\[
\|g_{\mathrm{sm}}(Y)-\bar g(Y)\|
\le
\beta\Big(D_w\,\delta_n(\tau,R_w)+D_l\,\delta_n(\tau,R_l)\Big)
\]
for every sampled pool $Y$. Consequently,
\[
V_{\mathrm{sm}}(x,\theta)
\le
2\operatorname{Var}(\bar g(Y))
+
2\beta^2\Big(D_w\,\delta_n(\tau,R_w)+D_l\,\delta_n(\tau,R_l)\Big)^2.
\]
\end{lemma}

\begin{proof}
Write
\[
\Xi(Y):=g_{\mathrm{sm}}(Y)-\bar g(Y).
\]
By \eqref{eq:sm_grad_exact},
\[
\Xi(Y)
=
\beta\,\phi'(\Delta_{\mathrm{sm}}(Y))
\left[
\left(\sum_{i=1}^n q_{w,i}(Y)u_i(Y)-\bar u(Y)\right)
- \left(\sum_{i=1}^n q_{l,i}(Y)v_i(Y)-\bar v(Y)\right)
\right].
\]
Since $\phi(u)=-\log \sigma(u)$, we have $|\phi'(u)|\le 1$.

For the preferred branch, Assumption~\ref{asm:score_range} implies
\[
\frac{\max_i e^{-\tau s_i^w(Y)}}{\min_i e^{-\tau s_i^w(Y)}}\le e^{\tau R_w}.
\]
Hence every soft weight satisfies
\[
\frac{1}{e^{\tau R_w}+n-1}
\le q_{w,i}(Y)
\le
\frac{e^{\tau R_w}}{e^{\tau R_w}+n-1}.
\]
Since $\max_i q_{w,i}(Y)\ge 1/n$, we have
\begin{align}
\operatorname{TV}(q_w(Y),\mathrm{Unif}_n)
&=
1-\sum_{i=1}^n \min\{q_{w,i}(Y),1/n\} \notag\\
&\le
1-\frac1n-\frac{n-1}{e^{\tau R_w}+n-1}
= \delta_n(\tau,R_w).
\label{eq:tv_bound_w}
\end{align}
Now let $z$ be any unit vector in the dual norm. The scalar function
\[
f_z(i):=\langle z,u_i(Y)\rangle
\]
has oscillation at most $D_w$ by Assumption~\ref{asm:grad_diameter}. Therefore, by the standard total-variation bound for bounded observables,
\[
\left|
\sum_{i=1}^n q_{w,i}(Y)\langle z,u_i(Y)\rangle
- \frac1n\sum_{i=1}^n \langle z,u_i(Y)\rangle
\right|
\le
D_w\,\operatorname{TV}(q_w(Y),\mathrm{Unif}_n).
\]
Taking the supremum over dual unit vectors $z$ and using \eqref{eq:tv_bound_w} gives
\[
\left\|\sum_{i=1}^n q_{w,i}(Y)u_i(Y)-\bar u(Y)\right\|
\le D_w\,\delta_n(\tau,R_w).
\label{eq:avg_dev_w}
\]

The non-preferred branch is identical. Assumption~\ref{asm:score_range} implies
\[
\operatorname{TV}(q_l(Y),\mathrm{Unif}_n)\le \delta_n(\tau,R_l),
\]
and Assumption~\ref{asm:grad_diameter} yields
\[
\left\|\sum_{i=1}^n q_{l,i}(Y)v_i(Y)-\bar v(Y)\right\|
\le D_l\,\delta_n(\tau,R_l).
\label{eq:avg_dev_l}
\]

Combining \eqref{eq:avg_dev_w} and \eqref{eq:avg_dev_l}, we obtain
\[
\|\Xi(Y)\|
\le
\beta\Big(D_w\,\delta_n(\tau,R_w)+D_l\,\delta_n(\tau,R_l)\Big).
\]
This proves the first claim.

For the variance bound, write
\[
g_{\mathrm{sm}}(Y)-\mathbb{E}g_{\mathrm{sm}}(Y)
=
\bigl(\bar g(Y)-\mathbb{E}\bar g(Y)\bigr)
+ \bigl(\Xi(Y)-\mathbb{E}\Xi(Y)\bigr).
\]
Using $\|a+b\|^2\le 2\|a\|^2+2\|b\|^2$ and
\[
\operatorname{Var}(\Xi)\le \mathbb{E}\|\Xi\|^2,
\]
we get
\[
V_{\mathrm{sm}}(x,\theta)
\le
2\operatorname{Var}(\bar g(Y))
+ 2\mathbb{E}\|\Xi(Y)\|^2.
\]
Applying the pointwise bound on $\|\Xi(Y)\|$ yields
\[
V_{\mathrm{sm}}(x,\theta)
\le
2\operatorname{Var}(\bar g(Y))
+ 2\beta^2\Big(D_w\,\delta_n(\tau,R_w)+D_l\,\delta_n(\tau,R_l)\Big)^2.
\]
\end{proof}

%----------------------------------------------------------------------
\subsection{Proof of Theorem~\ref{thm:variance_cpo}}
\label{sec:smooth_upper_proof}
%----------------------------------------------------------------------

\begin{proof}
Lemma~\ref{lem:hard_lower} gives
\[
V_{\mathrm{hard}}(x,\theta)\ge p^2m^2.
\]
Lemma~\ref{lem:smooth_upper} gives
\[
V_{\mathrm{sm}}(x,\theta)
\le
2\operatorname{Var}(\bar g(Y))
+ 2\beta^2\Big(D_w\,\delta_n(\tau,R_w)+D_l\,\delta_n(\tau,R_l)\Big)^2.
\]
Subtracting the two bounds yields
\[
V_{\mathrm{hard}}(x,\theta)-V_{\mathrm{sm}}(x,\theta)
\ge
p^2m^2
- 2\operatorname{Var}(\bar g(Y))
- 2\beta^2\Big(D_w\,\delta_n(\tau,R_w)+D_l\,\delta_n(\tau,R_l)\Big)^2.
\]
Assumption~\ref{asm:selector_dominance} therefore implies
\[
V_{\mathrm{hard}}(x,\theta)>V_{\mathrm{sm}}(x,\theta).
\]
\end{proof}

\section{Data Aggregation Strategies}
\label{sec:agg_method}
Table \ref{tab:agg} summarizes the different aggregation strategies used to aggregate multi-pair preference data for preference optimization. We present 4 strategies and corresponding formulation to generate pseudo preferred data $\tilde{y}_w$ and pseudo non-preferred data $\tilde{y}_l$.

% \begin{table}[h]
% \centering
% \small
% \renewcommand{\arraystretch}{1.6}
% \setlength{\tabcolsep}{6pt}
% \begin{tabular}{lccc}
% \specialrule{1.5pt}{0pt}{0pt}
% \textbf{Method} 
% & $\log P_\theta(\tilde{y}_w|x)$ 
% & $\log P_\theta(\tilde{y}_l|x)$ 
% & \textbf{Description} \\
% \midrule

% \textit{Weighted} 
% & $\displaystyle \sum_{i=1}^n \frac{P_{\theta_0}(y_w^i|x)}{\sum_{i=1}^n P_{\theta_0}(y_w^i|x)} \log P_\theta(y_w^i|x)$ 
% & $\displaystyle \sum_{i=1}^n \frac{P_{\theta_0}(y_l^i|x)}{\sum_{i=1}^n P_{\theta_0}(y_l^i|x)} \log P_\theta(y_l^i|x)$ 
% & Likelihood-weighted aggregation \\
% \midrule

% \textit{Average} 
% & $\displaystyle \frac{1}{n}\sum_{i=1}^n \log P_\theta(y_w^i|x)$ 
% & $\displaystyle \frac{1}{n}\sum_{i=1}^n \log P_\theta(y_l^i|x)$ 
% & Uniform averaging \\
% \midrule

% \textit{Min} 
% & $\displaystyle \min_{1\le i \le n} \log P_\theta(y_w^i|x)$ 
% & $\displaystyle \min_{1\le i \le n} \log P_\theta(y_l^i|x)$ 
% & Select worst likelihood \\
% \midrule

% \textit{Max} 
% & $\displaystyle \max_{1\le i \le n} \log P_\theta(y_w^i|x)$ 
% & $\displaystyle \max_{1\le i \le n} \log P_\theta(y_l^i|x)$ 
% & Select best likelihood \\
% \specialrule{1.5pt}{0pt}{0pt}

% \end{tabular}
% \caption{Different aggregation strategies for multi-pair preference data to formulate $\tilde{y}_w$ and $\tilde{y}_l$.}
% \label{tab:agg}
% \end{table}

\begin{table}[h]
\centering
\caption{Different aggregation strategies for multi-pair preference data to formulate $\tilde{y}_w$ and $\tilde{y}_l$.}
\label{tab:agg}
\small
\renewcommand{\arraystretch}{1.6}
\setlength{\tabcolsep}{6pt}
\resizebox{\textwidth}{!}{%
\begin{tabular}{lccc}
\specialrule{1.5pt}{0pt}{0pt}
\textbf{Method} 
& $\log P_\theta(\tilde{y}_w|x)$ 
& $\log P_\theta(\tilde{y}_l|x)$ 
& \textbf{Description} \\
\midrule

\textit{Weighted} 
& $\displaystyle \sum_{i=1}^n \frac{P_{\theta_0}(y_w^i|x)}{\sum_{i=1}^n P_{\theta_0}(y_w^i|x)} \log P_\theta(y_w^i|x)$ 
& $\displaystyle \sum_{i=1}^n \frac{P_{\theta_0}(y_l^i|x)}{\sum_{i=1}^n P_{\theta_0}(y_l^i|x)} \log P_\theta(y_l^i|x)$ 
& Weighted aggregation by likelihood given base model\\
\midrule

\textit{Average} 
& $\displaystyle \frac{1}{n}\sum_{i=1}^n \log P_\theta(y_w^i|x)$ 
& $\displaystyle \frac{1}{n}\sum_{i=1}^n \log P_\theta(y_l^i|x)$ 
& Uniform averaging \\
\midrule

% \textit{Min} 
% & 
% $\log P_\theta(y_{w^{*}}|x)$, where $w^{*}=\argmin_{i} P_{\theta_0}(y_w^i|x)$
% & $\displaystyle \log P_\theta(y_{l^{*}}|x)$, where $l^{*}=\argmin_{i} P_{\theta_0}(y_l^i|x)$ 
% & Select lowest likelihood given base model \\
% \midrule

% \textit{Max} 
% &
% $\log P_\theta(y_{w^{*}}|x)$, where $w^{*}=\argmax_{i} P_{\theta_0}(y_w^i|x)$
% & $\displaystyle \log P_\theta(y_{l^{*}}|x)$, where $l^{*}=\argmax_{i} P_{\theta_0}(y_l^i|x)$ 
% & Select highest likelihood given base model\\

\specialrule{1.5pt}{0pt}{0pt}

\end{tabular}%
}
\label{tab:agg_strategies}
\end{table}

\section{Baseline}
\label{sec:baseline}
\begin{itemize}
\item CoT \citep{wei2023chainofthoughtpromptingelicitsreasoning} instructs the LLM to generate answers accompanied by explicit reasoning steps.

\item ReAct \citep{yao2023reactsynergizingreasoningacting} integrates reasoning, action, and observation steps, concluding the generation process upon reaching a terminal state. The action step involves either generating a query to retrieve additional knowledge or finalizing the generation, while the observation step incorporates the retrieved knowledge documents.

\item SelfAsk \citep{press-etal-2023-measuring} involves generating follow-up questions, retrieving relevant information, and answering these follow-up questions. Each retrieval operation is based on the generated follow-up questions. When no additional follow-up questions are produced, the LLM provides the answer to the original question. We prepend the newly retrieved knowledge to the original question following the approach of Yoran et al. (2023).

\item BlendFilter \citep{wang2024blendfilter} combines internal and external knowledge to enhance retrieval quality. Additionally, it employs the LLM to filter out irrelevant information, thereby preventing the model from being misled.

\item InstructRAG \citep{wei2025instructraginstructingretrievalaugmentedgeneration} instructs the LLM to generate rationales that analyze the relevance between the query and the retrieved knowledge.

\item RoseRAG \citep{liu2025roserag} generates preference data through rejection sampling and optimizes the model using a single hard preference pair selected based on the model's likelihoods.
\end{itemize}

\section{Implementation Details}
\label{sec:impl}
We evaluate models with two small-sclae LLMs: Qwen2.5-1.5B-Instruct \citep{qwen2.5} and Gemma-2-2b-it \citep{team2024gemma}. We employ the state-of-the-art efficient retrieval method ColBERT v2 \citep{santhanam2022colbert} as implemented by \citep{khattab2023dspy}, which applies quantization to accelerate approximate nearest neighbor search. Our knowledge base comprises Wikipedia abstracts from the 2017 dump \citep{khattab2023demonstrate}. For experiments, we use Transformers \citep{wolf2020transformers} and TRL \citep{vonwerra2022trl} as backbone frameworks for fine-tuning and vLLM \citep{kwon2023efficientmemorymanagementlarge} for inference. For multi-pair aggregation, we implemented the trainer based on TRL library to apply our {\name} strategy and comparisons to other data aggregation strategies for three types of preference optimization algorithms (DPO, CPO, ORPO). 
% \footnote{\url{https://github.com/Tptrix29/robust-rag}}

In all experiments, we adopt a 3-shot in-context learning setting following the approach of \citep{shao2023enhancing}, with the count of retrieval documents set to 5 for all methods. The prompts for generating preference data are shown in Figure~\ref{fig:rationale_generation}, while those for question answering are presented in Figure~\ref{fig:benchmark_prompt}. We generate preference data using the first 10k training samples from HPQA, 2WMHQA and MSQQA while all data in MHRAGQA. For all backbones, we apply LoRA \citep{hu2021loralowrankadaptationlarge} with a rank of 16 and \texttt{lora\_alpha} = 32, targeting "\texttt{all-linear} modules. Experiments are conducted on eight 80G A100 NVIDIA GPUs. The learning rate is $1e^{-5}$ and number of training epochs is 2. 

\section{Qualitative Comparison of RoseRAG and {\name}}
Figure~\ref{fig:roserag_prism_comparison} presents a case study comparing the inference outputs of models trained with RoseRAG and {\name} on a multi-hop question with noisy retrieved passages. The RoseRAG-trained model, which was optimized using only a single hardest preference pair per query, fails to reason correctly through the noisy context and produces an incorrect answer. In contrast, the {\name}-trained model, which leveraged all preference pairs via smoothed aggregation during training, demonstrates stronger robustness to noisy distractors and arrives at the correct answer. This case illustrates that training with smoothed multi-pair aggregation yields models with improved evidence discrimination and multi-hop reasoning ability under noisy retrieval conditions.

\begin{figure}[h!]
\centering
  \includegraphics[width=0.95\columnwidth]{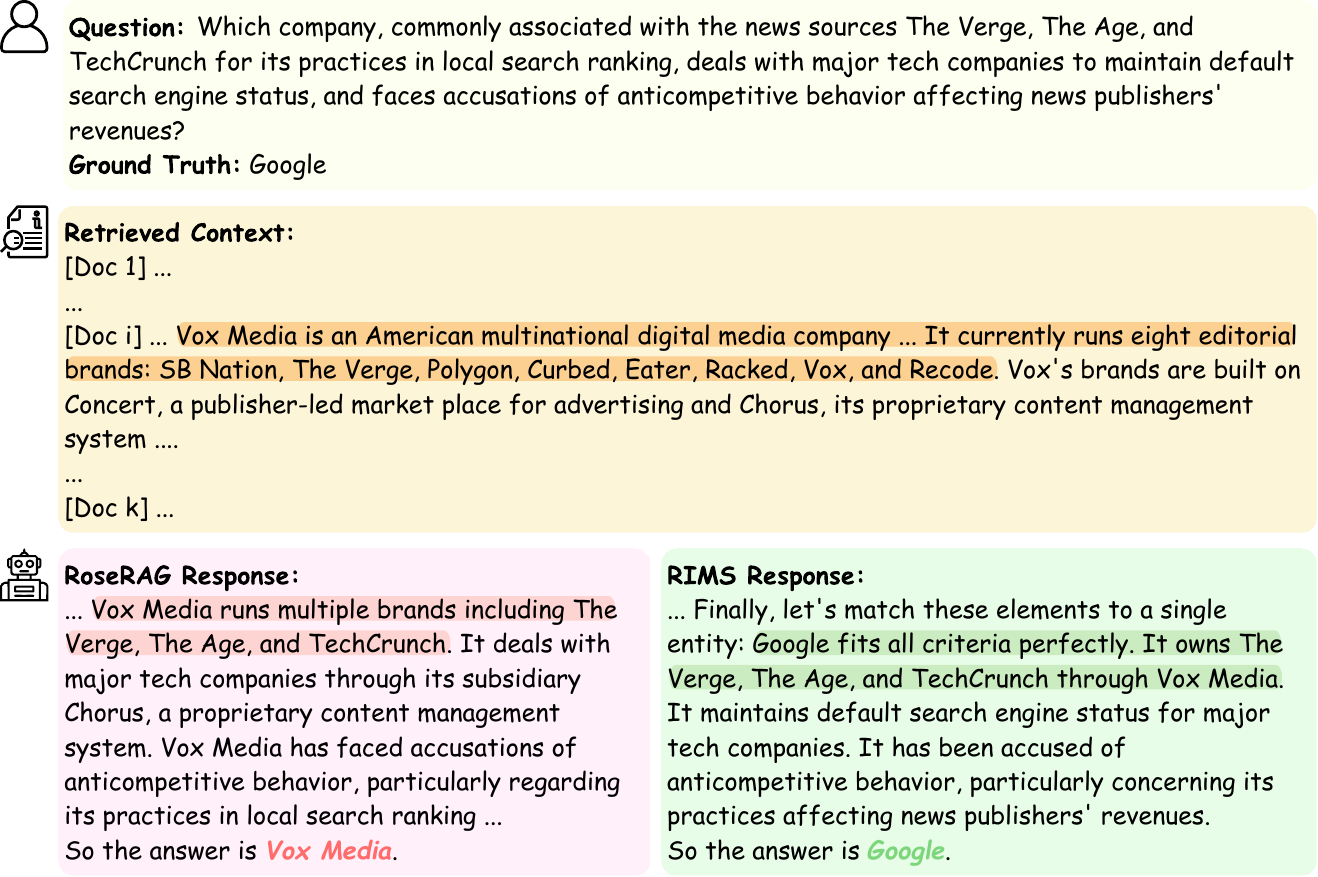}
  \caption{Qualitative comparison on a multi-hop question. RoseRAG, which relies on hard pair selection, is misled by noisy retrieved passages and produces an incorrect answer. {\name}, using soft aggregation over multiple preference pairs, arrives at the correct answer despite the same noisy context.}
  \label{fig:roserag_prism_comparison}
\end{figure}

\section{Prompt Details}
We show the prompts we used in Fig.~\ref{fig:rationale_generation} and Fig.~\ref{fig:benchmark_prompt}.
\label{sec:prompt}
% \begin{figure*}[t]
%     \centering
%     \includegraphics[width=2\columnwidth]{latex/figures/QA_Prompt.pdf}
%     \caption {Prompt for Question Answering Benchmark} 
%     \label{fig:qa_prompt}
% \end{figure*}
\begin{figure}[h]
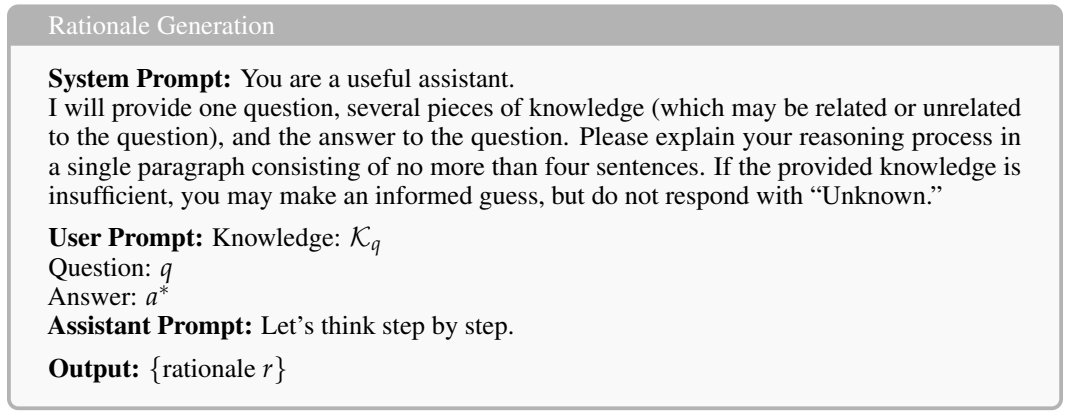

\centering
\begin{tcolorbox}[
    title=Rationale Generation, 
    colback=gray!5, 
    colframe=gray!60, 
]
\textbf{System Prompt:} You are a useful assistant.  

I will provide one question, several pieces of knowledge (which may be related or unrelated to the question), and the answer to the question. Please explain your reasoning process in a single paragraph consisting of no more than four sentences. If the provided knowledge is insufficient, you may make an informed guess, but do not respond with ``Unknown.''  

\vspace{0.5em}
\textbf{User Prompt:}  
Knowledge: $\mathcal{K}_q$

Question: $q$

Answer: $a^{*}$

\textbf{Assistant Prompt:} Let’s think step by step.  

\vspace{0.5em}
\textbf{Output:} $\{\text{rationale } r\}$
\end{tcolorbox}
\caption{Prompt template for rationale generation}
\label{fig:rationale_generation}
\end{figure}

\begin{figure}[h]
\centering
\begin{tcolorbox}[
  title=Few-shot Prompt for Question Answering Benchmark,
  colback=gray!5,
  colframe=gray!70,
  fonttitle=\bfseries
]

\textbf{System Prompt:}  
Y ou are a useful assistant. Y ou should answer the following question quickly and accurately based on given knowledge. Provide your reasoning in few sentences and answer in one or few words. Please begin your final answer with ’So the answer is’.

\vspace{0.8em}

\textbf{User Prompt:}  
Knowledge:Mr. Big (film) | Mr. Big is a 2007 documentary directed and produced by Tiffany Burns and edited by Alec MacNeill Richardson. The documentary examines the "Mr. Big" undercover methods used by the Royal Canadian Mounted Police (RCMP). In these operations, RCMP officers pose as gang criminals and develop a relationship with the target in the hope of eventually determining what, if any, knowledge the target has of the crime being investigated. "Mr. Big" operations have been credited with securing difficult convictions in a large number of cases, but concerns have been raised that they involve a risk of false confessions and wrongful convictions. It Might Get Loud | It Might Get Loud is a 2008 American documentary film by filmmaker Davis Guggenheim. It explores the careers and styles of prominent rock musicians Jimmy Page, The Edge, and Jack White. The film received a wide release on August 14, 2009 in the U.S. by Sony Pictures Classics.

Question:Are It Might Get Loud and Mr. Big both Canadian documentaries?

\vspace{0.4em}

\textbf{Assistant Prompt:} Let’s think step by step. Mr. Big is a 2007 documentary which examines the "Mr. Big" undercover methods used by the Royal Canadian Mounted Police. However, Are It Might Get Loud is a 2008 American documentary film. So the answer is no.

\vspace{0.8em}

\textbf{User Prompt:}  
Knowledge:Leslie H. Martinson | Leslie Herbert "Les" Martinson (January 16, 1915 – September 3, 2016) was an American television and film director. László Benedek | László Benedek (] ; March 5, 1905 – March 11, 1992; sometimes "Laslo Benedek") was a Hungarian-born film director and cinematographer, most notable for directing "The Wild One" (1953).

Question:Were László Benedek and Leslie H. Martinson both film directors?

\vspace{0.4em}

\textbf{Assistant Prompt:} Let’s think step by step. László Benedek was a Hungarian-born film director and Leslie H. Martinson was an American film director. So the answer is yes.

\vspace{0.8em}

\textbf{User Prompt:}  
Knowledge:Lucium | Lucium was the proposed name for an alleged new element found by chemist Prosper Barrière in 1896 in the mineral monazite. Later, William Crookes confirmed that the new element was actually an impure sample of yttrium. William Crookes | Sir William Crookes ( ; 17 June 1832 – 4 April 1919) was an English chemist and physicist who attended the Royal College of Chemistry in London, and worked on spectroscopy. He was a pioneer of vacuum tubes, inventing the Crookes tube which was made in 1875. Crookes was the inventor of the Crookes radiometer, which today is made and sold as a novelty item. Late in life, he became interested in spiritualism, and became the president of the Society for Psychical Research.

Question:Lucium was confimed to be an impure sample of yttrium by an English chemist who became the
president of what?

\vspace{0.4em}

\textbf{Assistant Prompt:} Let’s think step by step. Lucium was confimed to be an impure sample of yttrium by William Crookes. William Crookes is Sir William Crookes. Sir William Crookes became the president of the Society for Psychical Research. So the answer is Society for Psychical Research.

\vspace{0.8em}

\textbf{User Prompt:}  
Knowledge: $\mathcal{K}_q$

Question: $q$

\vspace{0.4em}

\textbf{Assistant Prompt:} Let’s think step by step. \{answer $a$\}

\end{tcolorbox}
\caption{Few-shot prompt template for benchmark}
\label{fig:benchmark_prompt}
\end{figure}

\end{document}